%% file: main_arxiv.tex
\documentclass{article}
\pdfoutput=1
\usepackage[utf8]{inputenc}
\usepackage[T1]{fontenc}
\usepackage{amsmath,amsthm, amssymb, bbm,color,mathtools}
\usepackage[verbose=true,
letterpaper,
textheight=8.4in,
textwidth=6.1in,
margin=1in,
headheight=12pt,
headsep=25pt,
footskip=30pt]{geometry}

\usepackage[toc,page,header]{appendix}
\usepackage{titletoc}
\usepackage[usenames,dvipsnames]{xcolor} 
\usepackage{hyperref}
\usepackage{ss}

\usepackage[round]{natbib}

\titlecontents{section}[3em]{\vspace{-1pt}}%
    {\bfseries\contentslabel{2em}}%
    {}%
    {\titlerule*[0.5pc]{.}\contentspage}%
    
\titlecontents{subsection}[5em]{\vspace{-1pt}}%
    {\contentslabel{2em}}%
    {}%
    {\titlerule*[0.5pc]{.}\contentspage}%

\titlecontents{subsubsection}[5em]{\vspace{-1pt}}%
    {\footnotesize\contentslabel{3em}}%
    {}%
    {\titlerule*[0.5pc]{.}\contentspage}%

\numberwithin{equation}{section}

\usepackage{url}            
\usepackage{booktabs}       
\usepackage{amsfonts}       
\usepackage{nicefrac}       
\usepackage{microtype}      
\usepackage{csquotes}

\usepackage{subcaption}
\usepackage{graphicx}
\usepackage{enumitem}
\usepackage{amsmath} 

\input{utils/math_commands.tex}

\usepackage[utf8]{inputenc} 
\usepackage[T1]{fontenc}    
\usepackage{hyperref}       
\usepackage{url}            
\usepackage{booktabs}       
\usepackage{amsfonts}       
\usepackage{nicefrac}       
\usepackage{microtype}      
\usepackage{calc}
\newlength{\maxmin}
\usepackage{mathtools}

\usepackage{xcolor}
\definecolor{TiffanyBlue}{HTML}{94D2BD}
\definecolor{Gamboge}{HTML}{EE9B00}
\definecolor{Rust}{HTML}{BB3E03}
\definecolor{DarkCyan}{HTML}{0A9396}
\definecolor{IndigoDye}{HTML}{003A52}
\definecolor{BlueViolet}{HTML}{8338EC}

\definecolor{Rufous}{HTML}{AE2012}
\definecolor{MidnightGreen}{HTML}{005F73}

\definecolor{FernGreen}{HTML}{588157}

\newcommand{\eigAone}{{\textcolor{MidnightGreen}{\lambda^A_1}}\xspace}
\newcommand{\eigAn}{{\textcolor{Rufous}{\lambda^A_n}}\xspace}

\newcommand{\eigVecAone}{{\textcolor{MidnightGreen}{\rvv^A_{1}}}\xspace}
\newcommand{\eigVecAn}{{\textcolor{Rufous}{\boldsymbol{1}}}\xspace}
\newcommand{\eigVecAnAbs}{{\textcolor{Rufous}{\rvv^A_n}}\xspace}

\newcommand\minput[1]{%
  \input{#1}%
  \ifhmode\ifnum\lastnodetype=11 \unskip\fi\fi}

\usepackage{graphicx}
\usepackage{subcaption}
\usepackage[most]{tcolorbox}
\usepackage{todonotes}
\usepackage{float}
\usepackage{enumitem}
\usepackage{xhfill}
\usepackage{wrapfig}
\usepackage{amsthm}

\RequirePackage{fancyhdr}
\RequirePackage{color}
\RequirePackage{algorithm}
\RequirePackage{algorithmic}
\RequirePackage{natbib}
\RequirePackage{eso-pic} 
\RequirePackage{forloop}

\definecolor{norange}{RGB}{249, 146, 0}
\definecolor{nred}{RGB}{234, 47, 32}
\definecolor{nblue}{RGB}{0, 100, 255}
\definecolor{npurple}{RGB}{146, 26, 192}
\usepackage{hyperref} 
\usepackage{cleveref}
\Crefname{section}{Sec.}{Secs.}
\Crefname{equation}{Eq.}{Eqs.}
\Crefname{figure}{Fig.}{Figs.}
\Crefname{tabular}{Tab.}{Tabs.}
\usepackage{soul}

\newtheorem{theorem}{Theorem}
\newtheorem{lemma}{Lemma}
\newtheorem{proposition}{Proposition}
\newtheorem{assumption}{Assumption}
\newtheorem{corollary}{Corollary}
\newtheorem{definition}{Definition}

\title{Setting the Record Straight on Transformer Oversmoothing} 

\author{Antiquus S.~Hippocampus, Natalia Cerebro \& Amelie P. Amygdale \thanks{ Use footnote for providing further information
about author (webpage, alternative address)---\emph{not} for acknowledging
funding agencies.  Funding acknowledgements go at the end of the paper.} \\
Department of Computer Science\\
Cranberry-Lemon University\\
Pittsburgh, PA 15213, USA \\
\texttt{\{hippo,brain,jen\}@cs.cranberry-lemon.edu} \\
\And
Ji Q. Ren \& Yevgeny LeNet \\
Department of Computational Neuroscience \\
University of the Witwatersrand \\
Joburg, South Africa \\
\texttt{\{robot,net\}@wits.ac.za} \\
\AND
Coauthor \\
Affiliation \\
Address \\
\texttt{email}
}

\usepackage{pifont}
\newcommand{\cmark}{\ding{51}\xspace}%
\newcommand{\xmark}{\ding{55}\xspace}%

\begin{document}

\linespread{1.03}              

\author{\name{Gb{\`e}tondji J-S Dovonon$^\dagger$} \email{gbetondji.dovonon.22@ucl.ac.uk}\\
\name{Michael Bronstein$^\S$}
\email{michael.bronstein@cs.ox.ac.uk}\\
\name{Matt J. Kusner$^{\dagger}$} \email{m.kusner@ucl.ac.uk}\\[2mm]
\addr{$^\dagger$University College London, UK}\\
\addr{$^\S$University of Oxford, UK}
}

\maketitle

\begin{abstract}
  Transformer-based models have recently become wildly successful across a diverse set of domains. At the same time, recent work has shown empirically and theoretically that Transformers are inherently limited. Specifically, they argue that as model depth increases, Transformers oversmooth, i.e., inputs become more and more similar. A natural question is: How can Transformers achieve these successes given this shortcoming? In this work we test these observations empirically and theoretically and uncover a number of surprising findings. We find that there are cases where feature similarity increases but, contrary to prior results, this is not inevitable, even for existing pre-trained models. Theoretically, we show that smoothing behavior depends on the eigenspectrum of the value and projection weights. We verify this empirically and observe that the sign of layer normalization weights can influence this effect. Our analysis reveals a simple way to parameterize the weights of the Transformer update equations to influence smoothing behavior. We hope that our findings give ML researchers and practitioners additional insight into how to develop future Transformer-based models.
\end{abstract}

\input{sections/1_introduction}
\input{sections/2_background_related}

\input{sections/3_not_destined_to_oversmooth}

\input{sections/4_reparameterization}

\section{Discussion}
In this paper, we unified past work on Transformer oversmoothing and tested these results empirically and theoretically. Empirically, we found that, contrary to prior findings, oversmoothing is not inevitable, even in existing pre-trained models. Theoretically, we presented a new analysis detailing how the eigenspectrum of attention and weight matrices influences smoothing behavior. Based on this we introduced a reparamterization of the Transformer weights that allows one to influence the smoothing behavior. This influence changes depending on the normalization scheme used. 
One limitation of the current theoretical analysis is that the results are asymptotic, applying in the limit as $\ell \rightarrow \infty$. 
It would be useful to understand the rates of convergence of each of the results. 
Alongside this we would like to expand the theoretical analysis to account for layer normalization and feed forward layers. Special conditions will likely need to be placed on $\rmH$ to enable this analysis, such as symmetric $\rmA,\rmH$ \citep{sander2022sinkformers}. We leave these extensions for future work.



\bibliography{references}
\bibliographystyle{icml2024}

\newpage
\appendix
\onecolumn



\section*{Appendix}

\section{Proofs}
\setcounter{proposition}{0}
\setcounter{theorem}{0}
\setcounter{lemma}{1}
\setcounter{corollary}{0}

\begin{proposition}[\cite{meyer2023matrix}]
Given Assumption~\ref{assume:A}, all eigenvalues of $\rmA$ lie within $(-1,1]$. There is one largest eigenvalue that is equal to $1$, with corresponding unique eigenvector $\boldsymbol{1}$. 
\end{proposition}

\begin{proof}
First, because $\rmA$ is positive, 
by the Perron-Frobenius Theorem \cite{meyer2023matrix} all eigenvalues of $\rmA$ are in $\mathbb{R}$ (and so there exist associated eigenvectors that are also in $\mathbb{R}$). Next, recall the definition of an eigenvalue $\lambda$ and eigenvector $\rvv$: $\rmA \rvv = \lambda \rvv$. Let us write the equation for any row $i \in \{1, \ldots, n\}$ explicitly:
\begin{align}
    a_{i1}v_1 + \cdots + a_{in}v_n = \lambda v_i. \nonumber
\end{align}
Further let,
\begin{align}
    v_{\max} := \max\{|v_1|, \ldots, |v_n|\}
\end{align}
Note that $v_{\max} > 0$, otherwise it is not a valid eigenvector. Further let $k_{\max}$ be the index of $\rvv$ corresponding to $v_{\max}$. Then we have,
\begin{align}
    |\lambda| v_{\max} =&\; |a_{k_{\max} 1}v_1 + \cdots + a_{k_{\max} n}v_n| \nonumber \\
    \leq&\; a_{k_{\max} 1}|v_{1}| + \cdots + a_{k_{\max} n}|v_{n}| \nonumber \\
    \leq&\; a_{k_{\max} 1}|v_{k_{\max}}| + \cdots + a_{k_{\max} n}|v_{k_{\max}}| \nonumber \\
    =&\; (a_{k_{\max} 1} + \cdots + a_{k_{\max} n})|v_{k_{\max}}| = |v_{\max}| \nonumber
\end{align}
The first inequality is given by the triangle inequality and because $a_{ij} > 0$. The second is given by the definition of $v_{\max}$ as the maximal element in $\rvv$. The final inequality is given by the definition of $\rmA$ in eq.~(\ref{eq:A}) as right stochastic (i.e., all rows of $\rmA$ sum to 1) and because $|v_{k_{\max}}| = |v_{\max}|$. Next, note that because $v_{\max} > 0$, it must be that $\lambda \leq 1$. Finally, to show that the one largest eigenvalue is equal to 1, recall by the definition of $\rmA$ in eq.~(\ref{eq:A}) that $\rmA \boldsymbol{1} = \boldsymbol{1}$, where $\boldsymbol{1}$ is the vector of all ones. So $\boldsymbol{1}$ is an eigenvector of $\rmA$, with eigenvalue $\lambda^*=1$. Because $a_{ij} > 0$, and we showed above that all eigenvalues must lie in in $[-1,1]$, by the Perron-Frobenius theorem \cite{meyer2023matrix} $\lambda^*=1$ is the Perron root. This means that all other eigenvalues $\lambda_i$ satisfy the following inequality $|\lambda_i| < \lambda^*$. Further $\boldsymbol{1}$ is the Perron eigenvector, and all other eigenvectors have at least one negative component, making $\boldsymbol{1}$ unique. Finally, because $\rmA$ is diagonalizable it has $n$ linearly independent eigenvectors. 
\end{proof}

We now prove a lemma that will allow us to prove Theorem~\ref{thm:eigenvalues_full_update}.

\begin{lemma}
\label{lemma:all_cases}
Consider the Transformer update in  eq.~(\ref{eq:vec_update}). Let $\{\lambda^A_i, \rvv^A_i\}_{i=1}^n$ and $\{\lambda^H_j, \rvv^H_j\}_{j=1}^r$ for $r \leq d$ be the eigenvalue and eigenvectors of $\rmA$ and $\rmH$. Let the eigenvalues (and associated eigenvectors) be sorted as follows, $\eigAone \leq \cdots \leq \eigAn$ and $|1 + \lambda^H_1| \leq \cdots \leq |1 + \lambda^H_r|$. Let $\varphi^H_1, \ldots, \varphi^H_r$ be the phases of $\lambda^H_1, \ldots, \lambda^H_r$. 
As the number of layers $L \rightarrow \infty$, one eigenvalue dominates the rest (multiple dominate if there are ties):
\begin{equation*}
\resizebox{\columnwidth}{!}{$
    \begin{cases}
        \begin{rcases}
            (1 + \lambda^H_r \eigAn) & \hspace{\maxmin} \text{if $|1 + \lambda^H_r \eigAn| \geq 1$} \\
            (1 + \lambda^H_{\min}\eigAone) & \hspace{\maxmin} \text{if $|1 + \lambda^H_r \eigAn| < 1$} \\
        \end{rcases} & \text{if $\eigAone > 0$} \\
        \begin{rcases}
            (1 + \lambda^H_r \eigAn) & \hspace{\maxmin} \text{if $|1 + \lambda^H_r \eigAn| > |1 + \lambda^H_{k} \eigAone|$} \\
            (1 + \lambda^H_k \eigAone) & \hspace{\maxmin} \text{if $|1 + \lambda^H_r \eigAn| < |1 + \lambda^H_{k} \eigAone|$} \\
        \end{rcases} & \text{if $\eigAone < 0, \varphi^H_r \in [-\frac{\pi}{2}, \frac{\pi}{2}]$} \\
        \begin{rcases}
            (1 + \lambda^H_r \eigAn) & \hspace{\maxmin} \text{if $|1 + \lambda^H_r \eigAn| > |1 + \lambda^H_r \eigAone|$} \\
            (1 + \lambda^H_r \eigAone) & \hspace{\maxmin} \text{if $|1 + \lambda^H_r \eigAn| < |1 + \lambda^H_r \eigAone|$} \\
        \end{rcases} & \text{if $\eigAone < 0, \varphi^H_r \in (\frac{\pi}{2}, \pi] \cup [-\pi, -\frac{\pi}{2})$}
    \end{cases}
$}
\end{equation*}
where $\lambda^H_{\min}$ be the eigenvalue of $\rmH$ with smallest magnitude and $\lambda^H_k$ is the eigenvalue with the largest index $k$ such that $\varphi^H_{k} \in (\pi/2, \pi] \cup [-\pi, -\pi/2)$.
\end{lemma}

\begin{proof}
Given Lemma~\ref{lem:eigenvalues}, 
the eigenvalues and eigenvectors of $(\rmI + \rmH \otimes \rmA)$ are equal to $(1 + \lambda^H_j\lambda^A_i)$ and  $\rvv^H_j \otimes \rvv^A_i$ for all $j \in \{1, ..., d\}$ and $i \in \{1, \ldots, n\}$. 
Recall that eigenvalues (and associated eigenvectors) are sorted in the following order $\lambda^A_1 \leq \cdots \leq \lambda^A_n$ and $|1+\lambda^H_1| \leq \cdots \leq |1+\lambda^H_d|$. 
Our goal is to understand the identity of the dominating eigenvalue(s) $\lambda^H_{j^*}\lambda^A_{i^*}$ for all possible values of $\lambda_H, \lambda_A$. 

First recall that $\lambda^A_i \in (-1,1]$ and $\eigAn = 1$. A useful way to view selecting $\lambda^H_j \lambda^A_i$ to maximize $|1 + \lambda^H_j \lambda^A_i|$ is as maximizing distance to $-1$. If (i), $\eigAone > 0$ then $\lambda^A_i$, for all $i \in \{1, \ldots, n-1\}$ always shrinks $\lambda^H_j$ to the origin and $\eigAn$ leaves it unchanged. Because of how the eigenvalues are ordered we must have that $|1 + \lambda^H_r| = |1 + \lambda^H_j \eigAn| \leq |1 + \lambda^H_r \eigAn| = |1 + \lambda^H_r|$. If $|1 + \lambda^H_r \eigAn| \geq 1$ then shrinking any $\lambda^H_i$ to the origin will also move it closer to $-1$. However, if $|1 + \lambda^H_r \eigAn| < 1$ then shrinking to the origin can move $\lambda^H_i$ farther from $-1$ than $|1 + \lambda^H_r \eigAn|$. The eigenvalue of $\rmH$ that can be moved farthest is the one with the smallest overall magnitude, defined as $\lambda^H_{\min}$. The eigenvalue of $\rmA$ that can shrink it the most is $\eigAone$. This completes the first two cases.

If instead (ii), $\eigAone < 0$ then it is possible to `flip' $\lambda^H_j$ across the origin, and so the maximizer depends on $\varphi^H_r$. If a) $\varphi^H_r \in [-\pi/2, \pi/2]$ then let $\lambda^H_{k}$ be the eigenvalue with the largest index $k$ such that $\varphi^H_{k} \in (\pi/2, \pi] \cup [-\pi, -\pi/2)$. It is possible that `flipping' this eigenvalue across the origin makes it farther away than $\lambda^H_r$, i.e., $|1 + \lambda^H_{k} \eigAone| > |1 + \lambda^H_r \eigAn|$. In this case $(1 +\lambda^H_{k} \eigAone)$ dominates, otherwise $(1 + \lambda^H_r \eigAn)$ dominates. If they are equal then both dominate. If instead b) $\varphi^H_r \in (\pi/2, \pi] \cup [-\pi, -\pi/2)$ then either $|1 + \lambda^H_r \eigAn| > |1 + \lambda^H_{j'} \lambda^A_{i'}|$ for all $j' \neq d$ and $i' \neq n$, and so $(1 +  \lambda^H_r \eigAn)$ dominates, or `flipping' $\lambda^H_r$ increases its distance from $-1$, and so $|1 + \lambda^H_r \eigAone| > |1 + \lambda^H_{j'} \lambda^A_{i'}|$ for all $j' \neq d$ and $i' \neq n$, and so $(1 + \lambda^H_r \eigAone)$ dominates. Because we cannot have that $|1 + \lambda^H_r \eigAn| = |1 + \lambda^H_r \eigAone|$ as $\eigAone > -1$ this covers all cases.
\end{proof}

Now we can prove Theorem~\ref{thm:eigenvalues_full_update}.

\begin{theorem}
Given the Transformer update in  eq.~(\ref{eq:vec_update}), let $\{\lambda^A_i\}_{i=1}^n$ and $\{\lambda^H_j\}_{j=1}^r$ for $r \leq d$ be the eigenvalues of $\rmA$ and $\rmH$. Let the eigenvalues be sorted as follows, $\lambda^A_1 \leq \cdots \leq \lambda^A_n$ and $|1 + \lambda^H_1| \leq \cdots \leq |1 + \lambda^H_r|$. As the number of layers $\ell \rightarrow \infty$, there are two types of dominating eigenvalues: \textbf{(1)} $(1 + \lambda^H_{j^*}\eigAn)$. and \textbf{(2)} $(1 + \lambda^H_{j^*}\eigAone)$
\end{theorem}

The proof follows immediately from Lemma~\ref{lemma:all_cases}.

\begin{theorem} 
Given the Transformer update in  eq.~(\ref{eq:vec_update}), if a single eigenvalue dominates, as the number of total layers $\ell \rightarrow \infty$, the feature representation $\rmX_\ell$ converges to one of two representations:
\textbf{(1)} If $(1 + \lambda^H_j \eigAn)$ dominates then,
\begin{align}
\rmX_\ell \rightarrow (1 + \lambda^H_j \eigAn)^\ell s_{j,n} \eigVecAn {\rvv_j^H}^\top,
\label{eq:oversmoothing_thm}
\end{align}
\textbf{(2)} If $(1 +  \lambda^H_j \eigAone)$ dominates then, 
\begin{align}
\rmX_\ell \rightarrow (1 + \lambda^H_j \eigAone)^\ell s_{j,1} \eigVecAone{\rvv_j^H}^\top
\label{eq:sharpening_thm}
\end{align}
where $\rvv^H,\rvv^A$ are eigenvalues of $\rmH,\rmA$ and $s_{j,i} := \langle {\rvv^Q}^{-1}_{j,i}, \mathsf{vec}(\rmX) \rangle$ and  ${\rvv^Q}^{-1}_{j,i}$ is row $ji$ in the matrix $\rmQ^{-1}$ (here $\rmQ$ is the matrix of eigenvectors of $(\rmI + \rmH \otimes \rmA)$). \textbf{(3)} If multiple eigenvalues have the same dominating magnitude, $\rmX_\ell$ converges to the sum of the dominating terms.
\end{theorem}

\begin{proof}
Recall that the eigenvalues and eigenvectors of $(\rmI + \rmH \otimes \rmA)$ are equal to $(1 + \lambda^H_j\lambda^A_i)$ and  $\rvv^H_j \otimes \rvv^A_i$ for all $j \in \{1, ..., d\}$ and $i \in \{1, \ldots, n\}$. This means,
\begin{align*}
    \mathsf{vec}(\rmX_\ell) = \sum_{i,j} (1 + \lambda^H_j\lambda^A_i)^\ell \langle{\rvv^Q}^{-1}_{j,i},  \mathsf{vec}(\rmX)\rangle (\rvv^H_j \otimes \rvv^A_i). 
\end{align*}
Recall that ${\rvv^Q}^{-1}_{j,i}$ is row $ji$ in the matrix $\rmQ^{-1}$, where $\rmQ$ is the matrix of eigenvectors $\rvv^H_j \otimes \rvv^A_i$. Further recall that ${\textcolor{Rufous}{\rvv^A_i}} = \eigVecAn$. As described in Theorem~\ref{thm:eigenvalues_full_update}, as $\ell \rightarrow \infty$ at least one of the eigenvalues pairs $\lambda^H_j\lambda^A_i$ will dominate the expression $(1 + \lambda^H_j\lambda^A_i)^\ell$, which causes $\mathsf{vec}(\rmX_L)$ to converge to the dominating term. Finally, we can rewrite, $\rvv_1 \otimes \rvv_2$ as $\mathsf{vec}(\rvv_2\rvv_1^\top)$. Now all non-scalar terms have $\mathsf{vec}(\cdot)$ applied, so we can remove this function everywhere to give the matrix form given in eq.~(\ref{eq:oversmoothing_thm}) and eq.~(\ref{eq:sharpening_thm}).
\end{proof}

\begin{corollary} 
If the residual connection is removed in the Transformer update, then the eigenvalues are of the form $(\lambda^H_{j}\lambda^A_{i})$. Further, $(\lambda^H_{j^*}\eigAn)$ is always a dominating eigenvalue, and $\rmX_\ell \rightarrow  (\lambda^H_{j^*} \eigAn)^\ell s_{j,n} \eigVecAn {\overline{\rvv}_{j^*}^H}^\top$ as $\ell \rightarrow \infty$, where ${\overline{\rvv}_{j^*}^H}$ is the sum of all eigenvectors with eigenvalue equal to the dominating eigenvalue $\lambda^H_{j^*}$.
\end{corollary}

\begin{proof}
The eigendecomposition of the Transformer update without the residual connection is:
\begin{align*}
    \mathsf{vec}(\rmX_\ell) = \sum_{i,j} (\lambda^H_j\lambda^A_i)^\ell \langle{\rvv^Q}^{-1}_{j,i},  \mathsf{vec}(\rmX)\rangle (\rvv^H_j \otimes \rvv^A_i). 
\end{align*}
In this case, $(\lambda^H_{j^*}\eigAn)$ is always a dominating eigenvalue because $|\eigAn| > |\lambda^A_i|$ for any $i \in \{1, \ldots, n-1\}$. This and the above eigendecomposition yields $\rmX_\ell \rightarrow  (\lambda^H_{j^*} \eigAn)^\ell s_{j,n} \eigVecAn {\overline{\rvv}_{j^*}^H}^\top$ as $\ell \rightarrow \infty$.
\end{proof}

\begin{theorem} 
Given the Transformer update  eq.~(\ref{eq:vec_update}), as the number of total layers $\ell \rightarrow \infty$, if \textbf{(1)} one eigenvalue $(1 + \lambda^H_j \eigAn)$ dominates, we have input convergence, angle convergence, and rank collapse. If \textbf{(2)} one eigenvalue $(1 + \lambda^H_j \eigAone)$ dominates, we do not have input convergence or angle convergence, but we do have rank collapse. If \textbf{(3)} multiple eigenvalues have the same dominating magnitude and: (a) there is at least one dominating eigenvalue $(1 + \lambda^H_{j^*} \lambda^A_{i^*})$ where $\lambda^A_{i^*} \neq \eigAn$, then we do not have input convergence or angle convergence, or (b) the geometric multiplicity of $\eigAone$ and $\lambda^H_{j^*}$ are both greater than 1, then we also do not have rank collapse.
\end{theorem}

\begin{proof}
If \textbf{(1)} one eigenvalue $(1 + \lambda^H_j \eigAn)$ dominates then we have that $\rmX_\ell \rightarrow (1 + \lambda^H_j \eigAn)^\ell s_{j,n} \eigVecAn {\rvv_j^H}^\top$. Therefore, $\rmX_\ell$ has all the same inputs which also implies angle convergence and rank collapse. If \textbf{(2)} one eigenvalue $(1 + \lambda^H_j \eigAone)$ dominates then we have that $\rmX_\ell \rightarrow (1 + \lambda^H_j \eigAone)^\ell s_{j,1} \eigVecAone{\rvv_j^H}^\top$. Therefore, we do not have input convergence. Further as $\eigVecAone$ can contain both positive an negative components we do not have angle convergence. However, $\rmX_\ell$ is rank one so we do have rank collapse. If \textbf{(3)} multiple eigenvalues have the same dominating magnitude and: (a) there is at least one dominating eigenvalue $(1 + \lambda^H_{j^*} \lambda^A_{i^*})$ where $\lambda^A_{i^*} \neq \eigAn$ then we do not have input convergence or rank convergence, as shown for case (2); if (b) the geometric multiplicity of $\eigAone$ and $\lambda^H_{j^*}$ are both greater than 1, then $\rmX_\ell$ converges to the sum of at least 2 rank-1 matrices which are not themselves linear combinations of each other. Therefore, $\mathsf{rank}(\rmX_\ell) \geq 2$.
\end{proof}

\begin{corollary}
\label{coro:no-residual-input-convergence}
If the residual connection is removed in the Transformer update, input convergence, angle convergence, and rank collapse are guaranteed.
\end{corollary}

\begin{proof}
Corollary~\ref{coro:no-residual-features} tells us that in this case $\rmX_\ell \rightarrow  (\lambda^H_{j^*} \eigAn)^\ell s_{j,n} \eigVecAn {\overline{\rvv}_{j^*}^H}^\top$ as $\ell \rightarrow \infty$, where ${\overline{\rvv}_{j^*}^H}$ is the sum of all eigenvectors with eigenvalue equal to the dominating eigenvalue $\lambda^H_{j^*}$. This matrix has all the same features and so we have input convergence, angle convergence, and rank collapse.
\end{proof}

\begin{corollary} 
If the eigenvalues of $\rmH$ fall within $[-1,0)$, then $(1 + \lambda^H_{j^*} \eigAone)$ dominates. If the eigenvalues of $\rmH$ fall within $(0,\infty)$, then $(1 + \lambda^H_{j^*} \eigAn)$ dominates. 
\end{corollary}

\begin{proof}
Let $\lambda^H_1 \leq \cdots \leq \lambda^H_r$. Again we can think of selecting $\lambda^H_j \lambda^A_i$ that maximizes $|1 + \lambda^H_j \lambda^A_i|$ as maximizing the distance of $\lambda^H_j \lambda^A_i$ to $-1$. Consider the first case where $\lambda^H_1, \cdots,  \lambda^H_r \in [-1,0)$, and so $\lambda^H_1$ is the closest eigenvalue to $-1$ and $\lambda^H_r$ is the farthest. If $\eigAone > 0$ then all $\lambda^A$ can do is shrink $\lambda^H$ to the origin, where $\eigAone$ shrinks $\lambda^H$ the most. The closest eigenvalue to the origin is $\lambda^H_r$, and so $(1 + \lambda^H_r \eigAone)$ dominates.
If instead $\eigAone < 0$, then we can `flip' $\lambda^H_j$ over the origin, making it farther from $-1$ than all other $\lambda^H_{j'}$. The eigenvalue that we can `flip' the farthest from $-1$ is $\lambda^H_{1}$, and so $(1 + \lambda^H_1 \eigAone)$ dominates. If all eigenvalues of $\rmH$ are equal, then both $(1 + \lambda^H_r \eigAone)$ and $(1 + \lambda^H_1 \eigAone)$ dominate.
For the second case where $\lambda^H_1, \cdots,  \lambda^H_r \in (0,\infty)$, we have that $|1 + \lambda^H_{r}\eigAn| > |1 + \lambda^H_{j'}\lambda^A_{i'}|$ for all $j' \in \{1, \ldots, d-1\}$ and $i' \in \{1, \ldots, n-1\}$. This is because, by definition $\lambda^H_{r}\eigAn > \lambda^H_{j'}\lambda^A_{i'}$. Further, $1 + \lambda^H_{r}\eigAn \geq |1 + \lambda^H_{j'}\lambda^A_{i'}|$ as the largest $|1 + \lambda^H_{j'}\lambda^A_{i'}|$ can be is either (i) $|1 - \epsilon\lambda^H_{r}|$ for $0 < \epsilon < 1$ or (ii) $|1 + \lambda^H_{r-1}\eigAn|$ (i.e., in (i) $\lambda^H_{r}$ is negated by $\eigAone$ and in (ii) $\lambda^H_{r-1}$ is the next largest value of $\lambda^H$). For (i), it must be that $1 + \lambda^H_{r}\eigAn \geq |1 - \epsilon\lambda^H_{r}|$ as $\lambda^H_{r} > 0$. For (ii) $\lambda^H_{r} \geq \lambda^H_{r-1} > 0$, and so $|1 + \lambda^H_{r}\eigAn| \geq |1 + \lambda^H_{r-1}\eigAn|$. Therefore $\eigAn$ dominates.
\end{proof}

\section{Training \& Architecture Details}

Crucially, even though our theoretical analysis applies for fixed attention $\rmA$ and weights $\rmH$, \textbf{we use existing model architectures throughout}, i.e., including different attention/weights each layer, multi-head attention, layer normalization (arranged in the pre-LN format \cite{xiong2020layer}), and fully-connected layers.\footnote{If a model has multiple heads we will define $\rmW_V = \rmV_H$ and $\rmW_{\mathsf{proj}} = \Lambda_H \rmV_H^\top$).} 


\paragraph{Image Classification: Training \& Architecture Details.}
We base our image classification experiments on the ViT model \cite{dosovitskiy2020image} and training recipe introduced in \cite{touvron2021training}. 
On CIFAR100 for 300 epochs using the cross-entropy loss and the AdamW optimizer \cite{adamw}. Our setup is the one used in \cite{park2022vision} which itself follows the DeiT training recipe \cite{touvron2021training}. We use a cosine annealing schedule with an initial learning rate of $1.25 \times 10^{-4}$ and weight decay of $5 \times 10^{-2}$. We use a batch size of 96. We use data augmentation including RandAugment \cite{cubuk2019randaugment}, CutMix \cite{yun2019cutmix}, Mixup \cite{zhang2018mixup}, and label smoothing \cite{touvron2021training}.
The models were trained on two Nvidia RTX 2080 Ti GPUs. 
On ImageNet, we use the original DeiT code and training recipe described above. Changes from CIFAR100 are that we use a batch size of 512 and train on a single Nvidia RTX 4090 GPU.

\paragraph{Text Generation: Training \& Architecture Details.}
We base our NLP experiments on \citet{geiping2023cramming}, using their code-base. Following this work we pre-train encoder-only `Crammed' Bert models with a maximum budget of 24 hours. We use a masked language modeling objective and train on the Pile dataset \cite{gao2020pile}. The batch size is 8192 and the sequence length is 128. We evaluate models on SuperGLUE \cite{wang2020superglue} after fine-tuning for each task. In order to ensure a fair comparison, all models are trained on a reference system with an RTX 4090 GPU. We use mixed precision training with bfloat16 as we found it to be the most stable \cite{kaddour2023train}.

\section{Distribution of the eigenvalues of $\rmH$ in trained models}

\input{figures/eigen_distrib}

\end{document}

%% file: utils/math_commands.tex
\usepackage{amsmath,amsfonts,bm}

\def\eqref#1{equation~\ref{#1}}

\def\1{\bm{1}}

\def\rvb{{\mathbf{b}}}

\def\rvm{{\mathbf{m}}}

\def\rvv{{\mathbf{v}}}

\def\rvx{{\mathbf{x}}}

\def\rmA{{\mathbf{A}}}

\def\rmD{{\mathbf{D}}}

\def\rmF{{\mathbf{F}}}

\def\rmH{{\mathbf{H}}}
\def\rmI{{\mathbf{I}}}

\def\rmM{{\mathbf{M}}}

\def\rmQ{{\mathbf{Q}}}

\def\rmU{{\mathbf{U}}}
\def\rmV{{\mathbf{V}}}
\def\rmW{{\mathbf{W}}}
\def\rmX{{\mathbf{X}}}

\DeclareMathAlphabet{\mathsfit}{\encodingdefault}{\sfdefault}{m}{sl}
\SetMathAlphabet{\mathsfit}{bold}{\encodingdefault}{\sfdefault}{bx}{n}



%% file: sections/1_introduction.tex
\section{Introduction}
\label{sec:introduction}

In recent years, Transformer models \cite{vaswani2017attention} have achieved astounding success across vastly different domains: e.g., vision \cite{dosovitskiy2020image,touvron2021training}, NLP \cite{touvron2023llama,wei2023chainofthought,kaddour2023challenges}, chemistry \cite{transformer_chem}, and many others.
However 
their performance can quickly saturate as model depth increases \cite{kaplan2020scaling,wang2022anti}. 
This appears to be caused by fundamental properties of Transformer models. Empirically, researchers first observed that as depth was increased, even to just 12 layers, features became more and more similar to one another \citep{tang2021augmented,zhou2021deepvit,zhou2021refiner,gong2021vision,yan2022addressing}. Theoretically, these observations were characterized as 
(a) \textbf{Input Convergence}: Transformer features converge to the exact same vector  \citep{park2022vision,wang2022anti,bai2022improving};
(b) \textbf{Angle Convergence}: the angle between Transformer features converges to 0 \citep{tang2021augmented,zhou2021deepvit,gong2021vision,yan2022addressing,shi2022revisiting,noci2022signal,guo2023contranorm}; 
or (c) \textbf{Rank Collapse}: Transformer features collapse to a rank one matrix \citep{dong2021attention,shi2022revisiting,noci2022signal,guo2023contranorm,ali2023centered}. In practice, this has led to a search for replacements for Transformer layers, including completely new attention blocks \cite{zhou2021deepvit,zhou2021refiner,wang2022anti,ali2023centered}, normalization layers \cite{guo2023contranorm,zhai2023stabilizing}, altered skip connections \cite{tang2021augmented,noci2022signal,shi2022revisiting}, convolutional layers \cite{park2022vision}, fully-connected layers \cite{liu2021pay,kocsis2022unreasonable,yu2022efficient}, and even average pooling layers \cite{yu2022metaformer}.

But are Transformers destined to oversmooth? In this work we test the above observations theoretically and empirically. Theoretically, we analyze the eigenspectrum of a simplified Transformer layer: fixed attention, weights, and a residual connection. We show even for this simplified setup that: (a) There are cases where all features converge to the same vector, but this is not inevitable, contrary to prior results; (b) Angle convergence is also possible, but not guaranteed; and (c) while rank collapse is likely, it is also not required. Empirically, for existing pre-trained models we find cases where (a) features do not converge to the same vector, (c) feature angles do not converge to $0$, and (c) rank does not collapse. In fact, our analysis uncovers a parameterization that allows one, in some cases better than others, to increase smoothing or reduce it. We observe that the sign of the weights of layer normalization plays a role in how much this parameterization influences smoothing behavior.

%% file: sections/2_background_related.tex
\section{Background \& Related Work}
\label{sec:background_related}

\subsection{The Transformer Update}
At their core, Transformers are a linear combination of a set of `heads'. Each head applies its own self-attention function on the input $\rmX \in \mathbb{R}^{n \times d}$ as follows
\begin{align}
\rmA := \mathsf{Softmax}\Big(\frac{1}{\sqrt{k}} \rmX \rmW_{Q} \rmW_{K}^\top \rmX^\top \Big), \label{eq:A}
\end{align}
where the $\mathsf{Softmax}(\cdot)$ function is applied to each row individually. Further, $\rmW_Q, \rmW_K \in \mathbb{R}^{d \times k}$ are learned query and key weight matrices. This `attention map' $\rmA$ then transforms the input to produce the output of a single head: $\rmA \rmX \rmW_{V} \rmW_{\mathsf{proj}}$, where $\rmW_V, \rmW_{\mathsf{proj}} \in \mathbb{R}^{d \times d}$ are learned value and projection weights. Most architectures then add a residual connection: 
\begin{align}
\rmX_{\ell} = \rmX_{\ell-1} +\rmA \rmX_{\ell-1} \rmW_{V} \rmW_{\mathsf{proj}}.
\label{eq:update}
\end{align}
These architectures also consist of layer-specific attention and weights, multiple heads (i.e., multiple $\rmA, \rmW_V$ are to $\rmX$ and the outputs of each head is summed), layer normalization (either in the Post-LN format \citep{vaswani2017attention,wang2019learning,xiong2020layer} e.g., for BERT \citep{kenton2019bert}, RoBERTa \citep{liu2019roberta}, and ALBERT \citep{lan2019albert}, or in the Pre-LN format \citep{baevski2018adaptive}, e.g., for GPT \citep{brown2020language}, ViT \citep{dosovitskiy2020image}, and PALM \citep{chowdhery2023palm} architectures), and fully-connected layers. Unfortunately, these layers make eigenspectrum analysis intractable (we detail why this is the case in Section~\ref{sec:not_destined_to_oversmooth}). However, recent work has demonstrated that simplified Transformer models have surprisingly similar behaviors as full models \citep{von2023transformers,mahankali2023one,ahn2024transformers,zhang2024trained,ahn2023linear}. We find that this is also the case for oversmoothing: even though our analysis considers the restricted update in eq.~(\ref{eq:update}) it can explain the smoothing behavior of full Transformer models (e.g., ViT and DeiT models).


\subsection{What Is Oversmoothing?}

In deep learning, `oversmoothing' broadly describes the tendency of a model to produce more and more similar features as depth increases. For Transformers, prior work largely uses one of three different ways to measure oversmoothing: (a) \textbf{Input Convergence}: Do the inputs converge to the exact same feature vector? \citep{park2022vision,wang2022anti,bai2022improving}; (b) \textbf{Angle Convergence}: Do the angles between inputs converge to 0? \citep{tang2021augmented,zhou2021deepvit,gong2021vision,yan2022addressing,shi2022revisiting,noci2022signal,guo2023contranorm}; (c) \textbf{Rank Collapse}: Does the rank of inputs collapse to 1? \citep{dong2021attention,shi2022revisiting,noci2022signal,guo2023contranorm,ali2023centered}. 

\paragraph{Input Convergence.}
One way to formalize oversmoothing is through the lens of signal-processing \citep{wang2022anti}: the smoothing of a function can be measured by how much it suppresses higher frequencies in the signal, removing smaller fluctuations to highlight the larger trend. To measure the smoothing of the Transformer update in eq.~(\ref{eq:update}) we can compute the ratio of high frequency signals to low frequency signals preserved in $\rmX_\ell$. If this goes to $0$ as $\ell \rightarrow \infty$, all high frequency information is lost: the signal is maximally smoothed. To estimate these signals we can compute the Discrete Fourier Transform (DFT) $\mathcal{F}$ of $\rmX_\ell$, via $\mathcal{F}(\rmX_\ell) := \rmF \rmX_\ell$, where $\rmF \in \mathbb{C}^{n \times n}$ is equal to $\rmF_{k,l} := e^{2\pi \mathrm{i} (k-1)(l-1)}$ for all $k,l \in \{2, \ldots, n\}$ (where $\mathrm{i} := \sqrt{-1}$), and is $1$ otherwise (i.e., in the first row and column). 
Define the Low Frequency Component (LFC) of $\rmX_\ell$ as $\mathrm{LFC}[\rmX_\ell] := \rmF^{-1}\mathrm{diag}([1, 0, \ldots, 0])\rmF \rmX_\ell = (1/n)\boldsymbol{1}\boldsymbol{1}^\top \rmX_\ell$. 
Further, define the High Frequency Component (HFC) of $\rmX_\ell$ as $\mathrm{HFC}[\rmX_\ell] := \rmF^{-1}\mathrm{diag}([0, 1, \ldots, 1])\rmF \rmX_\ell = (\rmI - (1/n)\boldsymbol{1}\boldsymbol{1}^\top) \rmX_\ell$. We can now state the first definition of oversmoothing:

\begin{definition}[Input Convergence \citep{wang2022anti}] \label{def:low-pass} The Transformer update in eq.~(\ref{eq:update}) oversmooths if for all $\rmX \in \mathbb{R}^{n \times d}$ we have that
\begin{align*}
    \lim_{\ell \rightarrow \infty} \frac{\|\mathrm{HFC}[\rmX_\ell]\|_2}{\|\mathrm{LFC}[\rmX_\ell]\|_2} = 0.
\end{align*}
\end{definition}

This definition measures the extent to which inputs converge to the same feature vector. To see this, notice that the term in the numerator $\mathrm{HFC}[\rmX_\ell] = (\rmI - (1/n)\boldsymbol{1}\boldsymbol{1}^\top) \rmX_\ell$ goes to 0 iff $\rmX_\ell = \boldsymbol{1}\overline{\rvx}^\top$ where $\overline{\rvx} \in \mathbb{R}^d$ is a vector where entry $\overline{x}_i$ is the mean of the $i$th column of $\rmX$. This is because $(1/n)\boldsymbol{1}\boldsymbol{1}^\top\rmX = \boldsymbol{1}\overline{\rvx}^\top$. Finally, the required condition $\rmX_\ell = \boldsymbol{1}\overline{\rvx}^\top$ only holds when all input vectors are equal. In the following we will refer to the ratio in the above definition as $\mathrm{HFC}/\mathrm{LFC}$.

\paragraph{Angle Convergence.}
Another way to quantify oversmoothing is via the cosine similarity between inputs:

\begin{definition}[Angle Convergence] \label{def:cosine_similarity}  
The Transformer update in eq.~(\ref{eq:update}) oversmooths if for all $\rmX \in \mathbb{R}^{n \times d}$ we have that
\begin{align*}
    \lim_{\ell \rightarrow \infty} \frac{2}{n(n-1)} \sum_{i=1}^n \sum_{j=i+1}^n \frac{\rvx_{i,\ell}^\top \rvx_{j,\ell}}{\|\rvx_{i,\ell}\|_2\|\rvx_{j,\ell}\|_2} = 1,
\end{align*}
\end{definition}
where $\rvx_{i,\ell} \in \mathbb{R}^d$ is the $i$th row of $\rmX_\ell$. This measures the cosine of the angle $\theta$ between every pair of inputs $\rvx_{i,\ell},\rvx_{j,\ell}$ and is 1 iff $\theta = 0$.

\paragraph{Rank Collapse.}
Finally, we can also measure oversmoothing via rank collapse in $\rmX_\ell$. This is usually described as $\lim_{\ell \rightarrow \infty} \mathrm{rank(\rmX_\ell)} = 1$. While rank can be computed via a singular value decomposition (SVD), it is highly-sensitive to the threshold deciding when a singular should be treated as zero. Instead, \citet{guo2023contranorm} use a continuous approximation of rank called the `effective rank', first introduced by \citet{roy2007effective}.

\begin{definition}[\textbf{Rank Collapse}] \label{def:rank_collapse} 
Given $\rmX_\ell \in \mathbb{R}^{n \times d}$, let $\rmX_\ell = \rmU_\ell \boldsymbol{\Sigma}_\ell \rmV_\ell$ be a singular value decomposition of $\rmX$ with singular values $\mathsf{diag}(\boldsymbol{\Sigma}_\ell) = [\sigma_{1,\ell}, \ldots, \sigma_{r,\ell}]$ for $r \leq \min\{n,d\}$ and $\sigma_{1,\ell} \geq \cdots \geq \sigma_{r,\ell} \geq 0$. Define the following discrete distribution according to the singular values as $p_{i,\ell} = \sigma_{i,\ell} / \sum_{j=1}^r \sigma_{j,\ell}$. The effective rank \citep{roy2007effective} is the exponential of the entropy of this distribution: $\exp(-\sum_{i=1}^r -p_{i,\ell} \log p_{i,\ell})$. The Transformer update in eq.~(\ref{eq:update}) oversmooths if for all $\rmX \in \mathbb{R}^{n \times d}$ we have that
\begin{align*}
    \lim_{\ell \rightarrow \infty} \exp(-\sum_{i=1}^r p_{i,\ell} \log p_{i,\ell}) = 1.
\end{align*}
\end{definition}
\citet{roy2007effective} prove that $1 \leq \exp(-\sum_{i=1}^r p_{i,\ell} \log p_{i,\ell}) \leq \mathrm{rank}(\rmX_\ell) \leq r$. 

Notice that Definitions~\ref{def:low-pass}-\ref{def:rank_collapse} are progressively relaxed, i.e., if an update satisfies an oversmoothing definition, it also satisfies any later definitions.


\subsection{Observations of Transformer Oversmoothing}

The term `oversmoothing' was first coined by \citet{li2018deeper} to describe how GNN node features become more similar with more rounds of message passing. A similar observation was made for Transformers by \citet{zhou2021deepvit}. They observed that as depth was increased, the cosine similarity among self-attention layers also increased. After this work many other works noticed that feature similarity in vision and language Transformers also increased with depth \cite{zhou2021refiner,gong2021vision,tang2021augmented,raghu2021vision,yan2022addressing,shi2022revisiting,wang2022anti,park2022vision,bai2022improving,choi2023graph} found. Multiple works around this time found that it was possible to improve vision Transformers by replacing self-attention layers with convolutional layers \citep{han2021transformer,liu2021swin,jiang2021token,touvron2021going,yuan2021tokens,park2022vision}.

\input{figures/oversmoothing_theory}


\subsection{The Theory of Transformer Oversmoothing}

Figure~\ref{fig:theory} shows current work on the theory of Transformer oversmoothing for three types of Transformer updates.

\paragraph{Input Convergence.}
\cite{wang2022anti} analyzed oversmoothing from the lens of signal processing (Definition~\ref{def:low-pass}). They showed that as the number of self-attention operations tended to infinity, all inputs converge to the same feature vector, producing a low-pass filter. They also analyzed the convergence rate when the residual connection, weights, multiple heads, and a linear layer is added, and found that convergence is not guaranteed. However, they argued that even with these additions oversmoothing still happens: `it is inevitable that high-frequency components are continuously diluted as ViT goes deeper', i.e., Definition~\ref{def:low-pass} holds. At the same time \citet{shi2022revisiting} analyzed oversmoothing using a different notion of input convergence. Curiously, while they argue that oversmoothing can be due to the parameters of layer normalization, their analysis seems to suggest that without layer normalization, oversmoothing does not occur. Because their focus is on the effect of normalization, which we do not consider here (more details on why we do not analyze layer normalization in Section~\ref{sec:not_destined_to_oversmooth}), we will describe input convergence using the definition of \citet{wang2022anti}. 

\paragraph{Angle Convergence.}
As far as we are aware there is no prior work that directly analyzes oversmoothing from the perspective of angle convergence (Definition~\ref{def:cosine_similarity}). However, if an update is shown to input-converge it will also angle-converge (and rank collapse), because input convergence is a stricter requirement than angle convergence (and rank collapse).

\paragraph{Rank Convergence.} 
The first work we are aware of that developed a theory around Transformer oversmoothing was \citet{dong2021attention} using the notion of rank collapse. Initially, they showed that, without skip-connections, repeated self-attention layers converge double-exponentially to a rank 1 matrix. They then show that there exist models where skip-connections counteract this convergence. 
\citet{noci2022signal} contradict this, arguing that oversmoothing still happens when the residual connection is added, but this can be counteracted if the residual connection is scaled appropriately. 
\citet{ali2023centered} show that rank collapse happens without a residual connection and value and projection weights. 
When a residual connection is added our analysis shows that it is possible to avoid rank collapse. 




%% file: figures/oversmoothing_theory.tex
\begin{figure*}[t!]
\centering
\includegraphics[width=\textwidth]{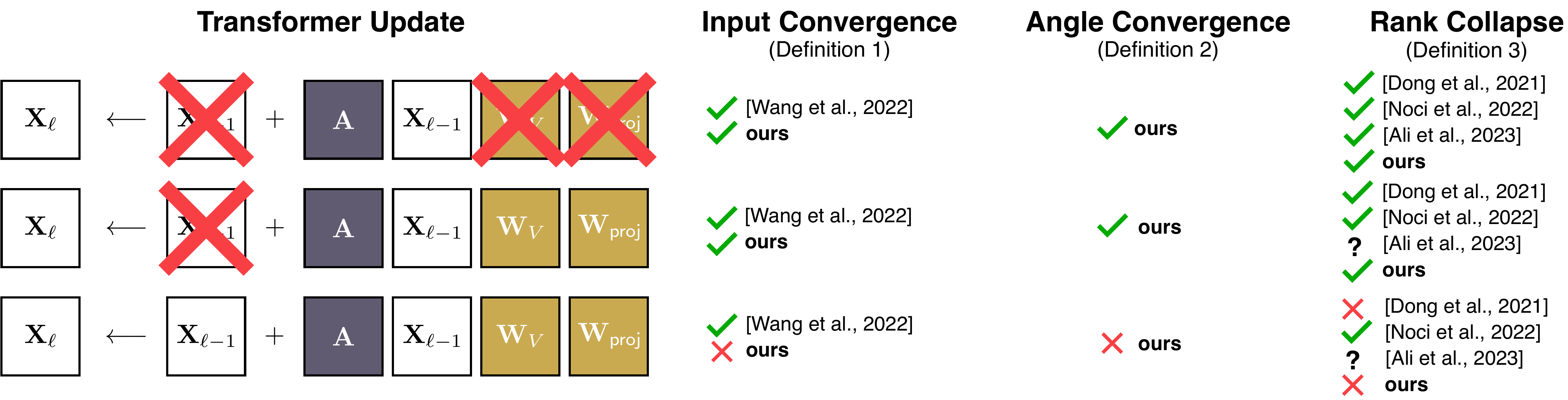}
\caption{\textbf{Theory of Transformer Oversmoothing.} A \cmark indicates prior work says that the corresponding Definition is always satisfied, an \xmark indicates it is not always satisfied. Note that if a work argues a Definition is satisfied, then all later Definitions, which are progressively more relaxed, must also be satisfied.}
\label{fig:theory}
\end{figure*}

%% file: sections/3_not_destined_to_oversmooth.tex
\section{Do Transformers Always Oversmooth?}
\label{sec:not_destined_to_oversmooth}

\begin{figure}[t!]
\centering
\includegraphics[width=\columnwidth]{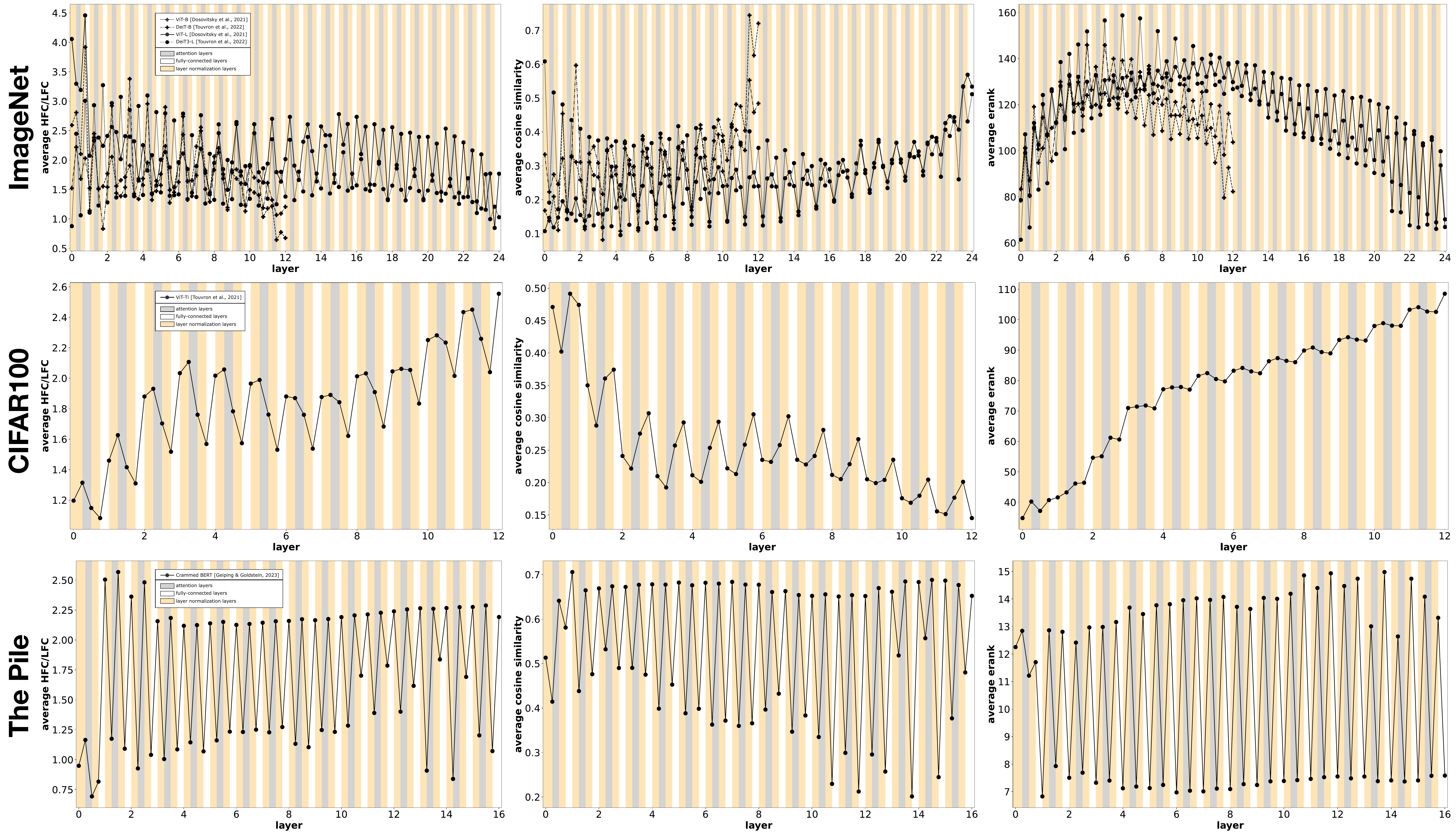}
\caption{\textbf{Smoothing behavior.} The smoothing metrics defined in Definitions~\ref{def:low-pass}-\ref{def:rank_collapse} for different models and datasets in vision and NLP. See text for details.}
\label{fig:base_smoothing}
\end{figure}

Given the current theory on Transformer oversmoothing, how are Transformer models so successful for vision and NLP applications \citep{kenton2019bert,liu2019roberta,lan2019albert,brown2020language,dosovitskiy2020image,chowdhery2023palm}? To investigate this, we computed the above three metrics in Definitions~\ref{def:low-pass}-\ref{def:rank_collapse} on a set of pre-trained models for vision and NLP that have been used in prior work on oversmoothing \citep{wang2022anti,choi2023graph} in Figure~\ref{fig:base_smoothing}. 
We notice that for all ImageNet models (ViT-B, ViT-L \citep{dosovitskiy2020image}, DeiT-B \citep{touvron2021training}, DeiT3-L \citep{Touvron2022DeiTIR}), as depth increases, we do see the metrics approaching their oversmoothing values as described in Definitions~\ref{def:low-pass}-\ref{def:rank_collapse}. Rank (Definition~\ref{def:rank_collapse}) does not consistently decrease and stays relatively high for 12 layer models, but continues to drop as depth is increased. 
However, we see something completely unexpected from the CIFAR model (ViT-Ti \citep{touvron2021training}). All of the metrics \emph{show reduction in smoothing behavior} as depth increases. Similarly, for The Pile model (Crammed BERT \citep{geiping2023cramming}) we see behavior that appears to oscillate between more and less smoothing. These behaviors motivate us to further investigate the Transformer update.

\subsection{Preliminaries}

Our strategy will be to understand the eigenspectrum of the Transformer update in the limit and to use this understanding to derive what the features $\rmX_\ell$ converge to as $\ell \rightarrow \infty$. This will allow us to understand if and when Definitions~\ref{def:low-pass}-\ref{def:rank_collapse} hold. We start by rewriting the Transformer update, eq.~(\ref{eq:update}), to make it more amenable to analysis. Define the $\mathsf{vec}(\rmM)$ operator as converting any matrix $\rmM$ to a vector $\rvm$ by stacking its columns. We can rewrite eq.~(\ref{eq:update}) vectorized as follows
\begin{align}
    \mathsf{vec}(\rmX_\ell) = (\rmI + \underbrace{\rmW_{\mathsf{proj}}^\top \rmW_{V}^\top}_{:= \rmH} \otimes \rmA)\mathsf{vec}(\rmX_{\ell-1}). \label{eq:vec_update}
\end{align}
This formulation is especially useful because $\mathsf{vec}(\rmX_\ell) = (\rmI + \rmH \otimes \rmA)^\ell\mathsf{vec}(\rmX_0)$. We now introduce an assumption on $\rmA$ that is also used in prior work \citep{ali2023centered,wang2022anti}.

\begin{assumption}[\citep{ali2023centered,wang2022anti}]\label{assume:A} The attention matrix is positive, i.e., $\rmA > 0$, and diagonalizable.
\end{assumption}

This assumption nearly always holds unless $\rmA$ numerically underflows. Initialization for $\rmW_Q,\rmW_K$ and normalization for $\rmX$ are often designed to avoid this scenario. In our experiments we never encountered $a_{ij} = 0$ for any element $(i,j) \in \mathbb{R}^n \times \mathbb{R}^n$ or $\rmA$ that was not diagonalizable, in any architecture. Note $\rmA$ is also right-stochastic, i.e., $\sum_{j} a_{i,j} = 1$, by definition in eq.~(\ref{eq:A}). This combined with Assumption~\ref{assume:A} immediately implies the following proposition.

\begin{proposition}[\cite{meyer2023matrix}]\label{prop:A} 
Given Assumption~\ref{assume:A}, all eigenvalues of $\rmA$ lie within $(-1,1]$. There is one largest eigenvalue that is equal to $1$, with corresponding unique eigenvector $\boldsymbol{1}$. 
\end{proposition}

We leave the proof to the Appendix. We can now analyze the eigenvalues of the Transformer update equations.

\subsection{The Eigenvalues}

First notice that the eigenvalues of $(\rmI + \rmH \otimes \rmA)^\ell$ can be written in terms of the eigenvalues of $\rmH,\rmA$:


\begin{lemma}
\label{lem:eigenvalues}
Let $\lambda^A_1, \ldots, \lambda^A_n$ be the eigenvalues of $\rmA$ and let $\lambda^H_1, \ldots, \lambda^H_r$ for $r \leq d$ be the eigenvalues of $\rmH$. The eigenvalues of $(\rmI + \rmH \otimes \rmA)^\ell$ are equal to $(1 + \lambda^H_j \lambda^A_i)$ for $j \in \{1,\ldots,r\}$ and $i \in \{1,\ldots,n\}$.
\end{lemma}

The proof can be derived from Theorem 2.3 of \citep{schacke2004kronecker}. Given this, notice that as the number of layers $\ell$ in the Transformer update eq.~(\ref{eq:vec_update}) increases, one eigenvalue $(1 + \lambda^H_{j^*} \lambda^A_{i^*})$ will dominate the rest (except in cases of ties). 

\begin{definition} [Dominating eigenvalue(s)]
\label{def:dominating_eigenvalue}
At least one of the eigenvalues of $(\rmI + \rmH \otimes \rmA)$ has a larger magnitude than all others, i.e., there exists $j^*,i^*$ (which may be a set of indices if there are ties) such that $|1 + \lambda^H_{j^*}\lambda^A_{i^*}| > |1 + \lambda^H_{j'}\lambda^A_{i'}|$ for all $j' \in \{1, \ldots, r\} \setminus j^*$ and $i' \in \{1, \ldots, n\} \setminus i^*$. These eigenvalues are called \textbf{dominating}.
\end{definition}

Which eigenvalue dominates will control the smoothing behavior of the Transformer.




\begin{theorem}
\label{thm:eigenvalues_full_update}
Given the Transformer update in  eq.~(\ref{eq:vec_update}), let $\{\lambda^A_i\}_{i=1}^n$ and $\{\lambda^H_j\}_{j=1}^r$ for $r \leq d$ be the eigenvalues of $\rmA$ and $\rmH$. Let the eigenvalues be sorted as follows, $\lambda^A_1 \leq \cdots \leq \lambda^A_n$ and $|1 + \lambda^H_1| \leq \cdots \leq |1 + \lambda^H_r|$. As the number of layers $\ell \rightarrow \infty$, there are two types of dominating eigenvalues: \textbf{(1)} $(1 + \lambda^H_{j^*}\eigAn)$. and \textbf{(2)} $(1 + \lambda^H_{j^*}\eigAone)$
\end{theorem}

We leave the proof to the Appendix (where we describe all possible cases). We can now use this result to derive what $\rmX_\ell$ converges to as depth increases.



\subsection{The Features}

\begin{theorem} \label{thm:representation}
Given the Transformer update in  eq.~(\ref{eq:vec_update}), if a single eigenvalue dominates, as the number of total layers $\ell \rightarrow \infty$, the feature representation $\rmX_\ell$ converges to one of two representations:
\textbf{(1)} If $(1 + \lambda^H_j \eigAn)$ dominates then,
\begin{align}
\rmX_\ell \rightarrow (1 + \lambda^H_j \eigAn)^\ell s_{j,n} \eigVecAn {\rvv_j^H}^\top,
\label{eq:oversmoothing}
\end{align}
\textbf{(2)} If $(1 +  \lambda^H_j \eigAone)$ dominates then, 
\begin{align}
\rmX_\ell \rightarrow (1 + \lambda^H_j \eigAone)^\ell s_{j,1} \eigVecAone{\rvv_j^H}^\top
\label{eq:sharpening}
\end{align}
where $\rvv^H,\rvv^A$ are eigenvalues of $\rmH,\rmA$ and $s_{j,i} := \langle {\rvv^Q}^{-1}_{j,i}, \mathsf{vec}(\rmX) \rangle$ and  ${\rvv^Q}^{-1}_{j,i}$ is row $ji$ in the matrix $\rmQ^{-1}$ (here $\rmQ$ is the matrix of eigenvectors of $(\rmI + \rmH \otimes \rmA)$). \textbf{(3)} If multiple eigenvalues have the same dominating magnitude, $\rmX_\ell$ converges to the sum of the dominating terms.
\end{theorem}

\begin{corollary} 
\label{coro:no-residual-features} 
If the residual connection is removed in the Transformer update, then the eigenvalues are of the form $(\lambda^H_{j}\lambda^A_{i})$. Further, $(\lambda^H_{j^*}\eigAn)$ is always a dominating eigenvalue, and $\rmX_\ell \rightarrow  (\lambda^H_{j^*} \eigAn)^\ell s_{j,n} \eigVecAn {\overline{\rvv}_{j^*}^H}^\top$ as $\ell \rightarrow \infty$, where ${\overline{\rvv}_{j^*}^H}$ is the sum of all eigenvectors with eigenvalue equal to the dominating eigenvalue $\lambda^H_{j^*}$.
\end{corollary}

See the Appendix for proofs of the above statements. Given these results, we can now understand when the oversmoothing definitions apply.

\subsection{When Oversmoothing Happens}


\begin{theorem} \label{thm:smoothing}
Given the Transformer update  eq.~(\ref{eq:vec_update}), as the number of total layers $\ell \rightarrow \infty$, if \textbf{(1)} one eigenvalue $(1 + \lambda^H_{j^*} \eigAn)$ dominates, we have input convergence, angle convergence, and rank collapse. If \textbf{(2)} one eigenvalue $(1 + \lambda^H_{j^*} \eigAone)$ dominates, we do not have input convergence or angle convergence, but we do have rank collapse. If \textbf{(3)} multiple eigenvalues have the same dominating magnitude and: (a) there is at least one dominating eigenvalue $(1 + \lambda^H_{j^*} \lambda^A_{i^*})$ where $\lambda^A_{i^*} \neq \eigAn$, then we do not have input convergence or angle convergence, if also (b) the geometric multiplicity of $\eigAone$ and $\lambda^H_{j^*}$ are both greater than 1, then we also do not have rank collapse.
\end{theorem}

\begin{corollary}
\label{coro:no-residual-input-convergence}
If the residual connection is removed in the Transformer update, input convergence, angle convergence, and rank collapse are guaranteed.
\end{corollary}

The proofs are left to the Appendix. 
The above statements follow directly from Theorem~\ref{thm:representation} and Corollary~\ref{coro:no-residual-features}. They tell us that whenever a single eigenvalue $(1 + \lambda^H_j \eigAn)$ dominates, \emph{every input in $\rmX_\ell$ converges to the same feature vector}. This happens because $\eigVecAnAbs = \eigVecAn$ and so $\rvx_{\ell,i} \sim \rvv^H_j$, for all $i$ as $\ell \rightarrow \infty$. But there is a second case: whenever the single eigenvalue $(1 + \lambda^H_j \eigAone)$ dominates, each feature is not guaranteed to be identical. However, $\rmX_\ell \rightarrow (1 + \lambda^H_j \eigAone)^\ell s_{j,1} \eigVecAone{\rvv_j^H}^\top$ is still a matrix of rank one. If instead multiple eigenvalue dominate and the geometric multiplicity of $\eigAone$ and $\lambda^H_{j^*}$ are both greater than 1 then $\rmX_\ell$ is a sum of at least 2 rank-1 matrices and so we do not have rank collapse.

Theorem~\ref{thm:smoothing} largely contradicts prior theoretical results on oversmoothing. We suspect a few reasons for this. First, if multiple types of analyses are used within one paper, and they give conflicting results, resolving this can be especially challenging \citep{wang2022anti}. Second, certain assumptions may not always hold in practice, e.g., \citet{noci2022signal} assume that $\rmA = \frac{1}{n} \boldsymbol{1}\boldsymbol{1}^\top$ at initialization.

\paragraph{On Layer Normalization \& Feed Forward Layers.}
Most Transformers also include layer normalization and feedforward layers. Unfortunately, both of these break our analysis. For instance, a repeated Pre-LN layer can be represented by the following update,
\begin{align*}
    \mathsf{vec}(\rmX_\ell) = (\rmI + \rmH \rmD^{-1} \otimes \rmA)^\ell\mathsf{vec}(\rmX_{0}) - \ell(\mathsf{vec}(\rmA \boldsymbol{1}\rvb^\top\rmD^{-1}\rmH^\top)), 
\end{align*}
where $\rvb$ and $\rmD^{-1}$ are terms introduced by the normalization layer. However, as far as we are aware there is no way to characterize the relationship between the eigenvalues of $(\rmI + \rmH \rmD^{-1} \otimes \rmA)$ and the eigenvalues of $\rmH$, $\rmA$, and $\rmD$, without introducing further assumptions (e.g., if $\rmH$ is symmetric there is a known relationship). This difficulty also applies to Post-LN layers. We encounter a similar difficulty for feed forward layers,
\begin{align*}
\mathsf{vec}(\rmX_\ell) = (\rmW^\top \otimes \rmI + \rmW^\top \rmH \otimes \rmA)^\ell\mathsf{vec}(\rmX_{0}), 
\end{align*}
where $\rmW$ are the parameter of the feed forward layer. Similar to layer normalization, as far as we are aware, we cannot characterize the eigenvalues of $(\rmW^\top \otimes \rmI + \rmW^\top \rmH \otimes \rmA)$ in terms of the eigenvalues of $\rmH$, $\rmA$, and $\rmW$, without further assumptions.

A natural question is can we use the above analysis to influence the smoothing behavior of Transformer models? In the next section we derive a Corollary of Theorem~\ref{thm:eigenvalues_full_update} that allows one to do so using a simple reparameterization of $\rmH$.

%% file: sections/4_reparameterization.tex
\label{sec:reparameterization}
\section{A Reparameterization that Influences Smoothing}

\begin{figure}[t!]
\centering
\includegraphics[width=\columnwidth]{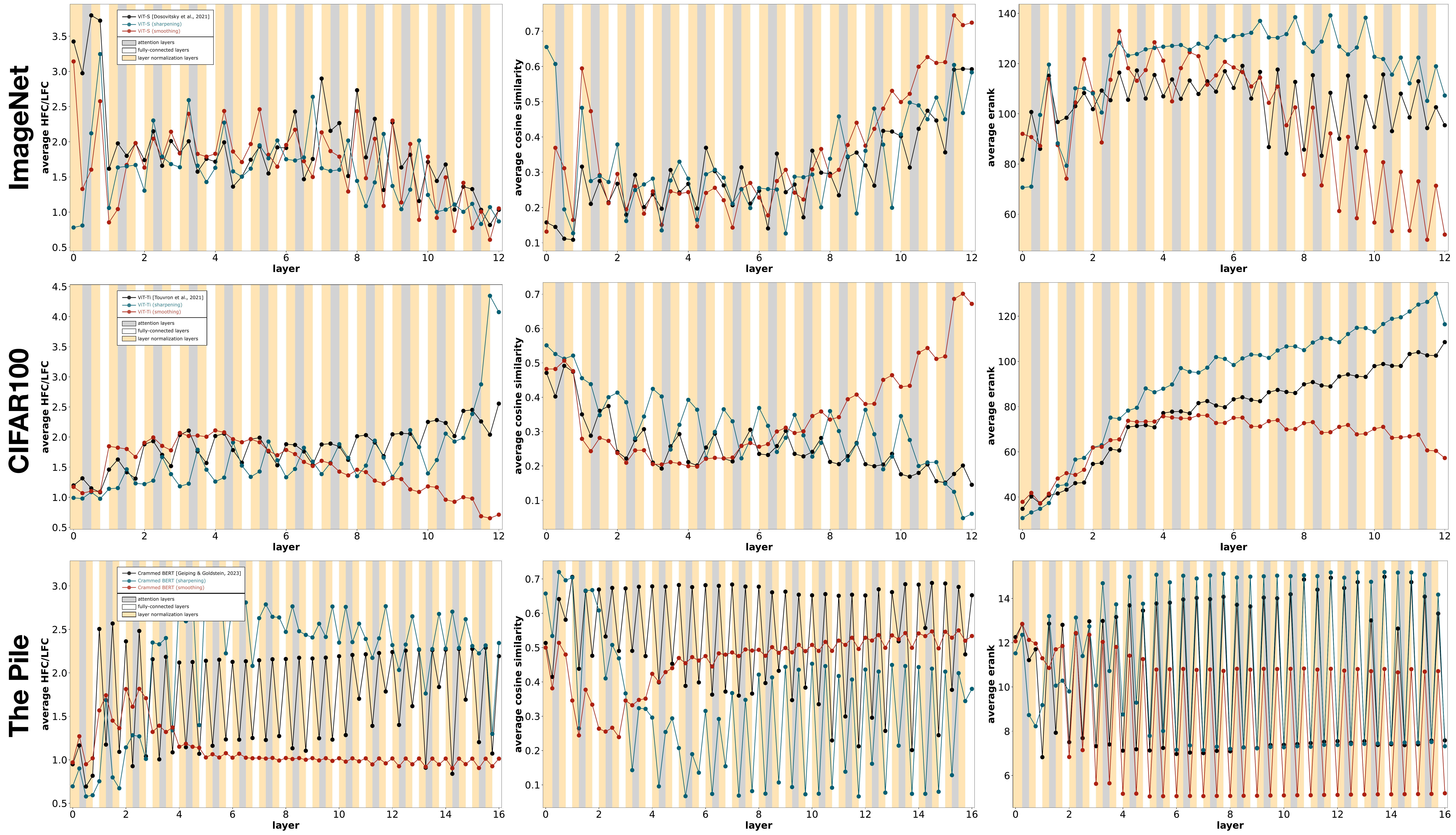}
\caption{\textbf{Influencing smoothing.} The smoothing metrics defined in Definitions~\ref{def:low-pass}-\ref{def:rank_collapse} for different models and datasets when $\rmH$ is reparameterized as $\rmH = \rmV_H \Lambda_H \rmV^{-1}_H$. See text for details.}
\label{fig:reparameterization}
\end{figure}

Theorem~\ref{thm:smoothing} tells us that if $(1 + \lambda_j^H \eigAn)$ dominates then this will cause, whereas if instead $(1 + \lambda_j^H \eigAone)$ dominates we avoid it. To find $\rmA$ and $\rmH$ that are guaranteed to have either $(1 + \lambda_j^H \eigAone)$ or $(1 + \lambda_j^H \eigAn)$ dominate we could dig through the proof of Theorem~\ref{thm:eigenvalues_full_update} and consider all cases. However, in practice, $\rmA$ changes for every batch of data $\rmX$. Because of this, we would like to find a solution that involves only controlling the eigenvalues of $\rmH$. Luckily, we can simplify the proof of Theorem~\ref{thm:eigenvalues_full_update} into a much simpler condition.


\begin{corollary} \label{coro:sharpening}
If the eigenvalues of $\rmH$ fall within $[-1,0)$, then $(1 + \lambda^H_{j^*} \eigAone)$ dominates. If the eigenvalues of $\rmH$ fall within $(0,\infty)$, then $(1 + \lambda^H_{j^*} \eigAn)$ dominates. 
\end{corollary}

See the Appendix for a proof. 
To ensure that the eigenvalues of $\rmH$ fall in these ranges, we propose to directly parameterize its eigendecomposition. Specifically, define $\rmH$ as $\rmH = \rmV_H \Lambda_H \rmV^{-1}_H$, where $\rmV_H$ is a full-rank matrix and $\Lambda_H$ is diagonal. We learn parameters $\rmV_H$ by taking gradients in the standard way (i.e., directly and through the inversion). 
To learn the diagonal of $\Lambda_H$, i.e., $\mathsf{diag}(\Lambda_H)$, we parameterize the sharpening model as \textcolor{MidnightGreen}{$\mathsf{diag}(\Lambda_H) := \mathsf{clip}(\boldsymbol{\psi}, [-1, 0])$}, where $\boldsymbol{\psi}$ are tunable parameters and $\mathsf{clip}(\boldsymbol{\psi}, [l,u]) := \min(\max(\boldsymbol{\psi},l),u)$ forces all of $\boldsymbol{\psi}$ to lie in $[l,u]$. Similarly we parameterize the smoothing model as \textcolor{Rufous}{$\mathsf{diag}(\Lambda_H) := \mathsf{clip}(\boldsymbol{\psi}, [0, 1])$}.\footnote{While we could have allowed the smoothing model to use the space of positive reals via \textcolor{Rufous}{$\mathsf{diag}(\Lambda_H) := |\boldsymbol{\psi}|$} , we found that restricting the space of allowed eigenvalues stabilized training.} 

\input{tables/hfclfc_change}
\input{tables/erank_change}
\input{tables/cosine_similarity_change}

\paragraph{Initialization.}
We initialize $\rmH = \rmV_H \Lambda_H \rmV^{-1}_H$ to mimic the initializations used in the ViT-Ti and Bert baselines, which are initialized using He initialization \cite{he2015delving}. Specifically, we first initialize $\rmV_H$ using He initialization. To initialize $\mathsf{diag}(\Lambda_H)$ we sample from a normal distribution with mean 0, as randomly initialized matrices will typically have normally distributed eigenvalues centered at 0. We noticed that if we set the standard deviation of this normal distribution to 1, the sampled values of $\mathsf{diag}(\Lambda_H)$ are often too large and lead to training instability. To stabilize training, we set the standard deviation to 0.1. All other training and architecture details are in the Appendix.

\paragraph{Reparameterization results.}
Figure~\ref{fig:reparameterization} show the effect of reparameterizing $\rmH$ and restricting the range of eigenvalues to encourage \textcolor{MidnightGreen}{sharpening} and \textcolor{Rufous}{smoothing}. For ImageNet we see that this does not have a large effect of $\mathrm{HFC}/\mathrm{LFC}$ and cosine similarity, but influences the effective rank somewhat in later layers. For CIFAR100 the \textcolor{MidnightGreen}{sharpening} parameterization reduces smoothing in all metrics while the \textcolor{Rufous}{smoothing} parameterization further increases smoothing. For The Pile the \textcolor{MidnightGreen}{sharpening} parameterization has little effect on $\mathrm{HFC}/\mathrm{LFC}$ and effective rank, but seems to reduce smoothing somewhat in terms of cosine similarity. The opposite is true of the \textcolor{Rufous}{smoothing} parameterization: little effect on cosine similarity, but increased smoothing for $\mathrm{HFC}/\mathrm{LFC}$ and effective rank.

\paragraph{Impact of layer normalization.}
The position and weights of the layer normalization layer can impact the filtering behavior of a layer. In \cref{fig:ln_impact} we parameterize two layers, one smoothing and one sharpening and apply it to an input image for 128 iterations in order to visualize its asymptotic behavior. We repeat the process with the two most common layer normalization implementations: Pre-LN and Post-LN \cite{xiong2020layer}, each with a positive then negative weight matrix sampled randomly. We do not use a bias since our focus is showing the impact of the normalization weight. When the weights are negative, Pre-LN reverses the expected filtering behavior of the layer. That is due to the normalization happening after the attention but before the residual connection. With Post-LN, the attention and residual connection are both applied before the normalization, so we still observe the expected behavior though it tends to be unstable for sharpening layers. This also means that while our reparametrization gives us control over the filtering behavior of the attention layers in a model, we can lose control when the layer normalization weights become negative. This could explain the surprising results we observe in \cref{tab:hfclfcchange} and \cref{fig:reparameterization}.

\begin{figure}[t!]
\centering
\includegraphics[width=\columnwidth]{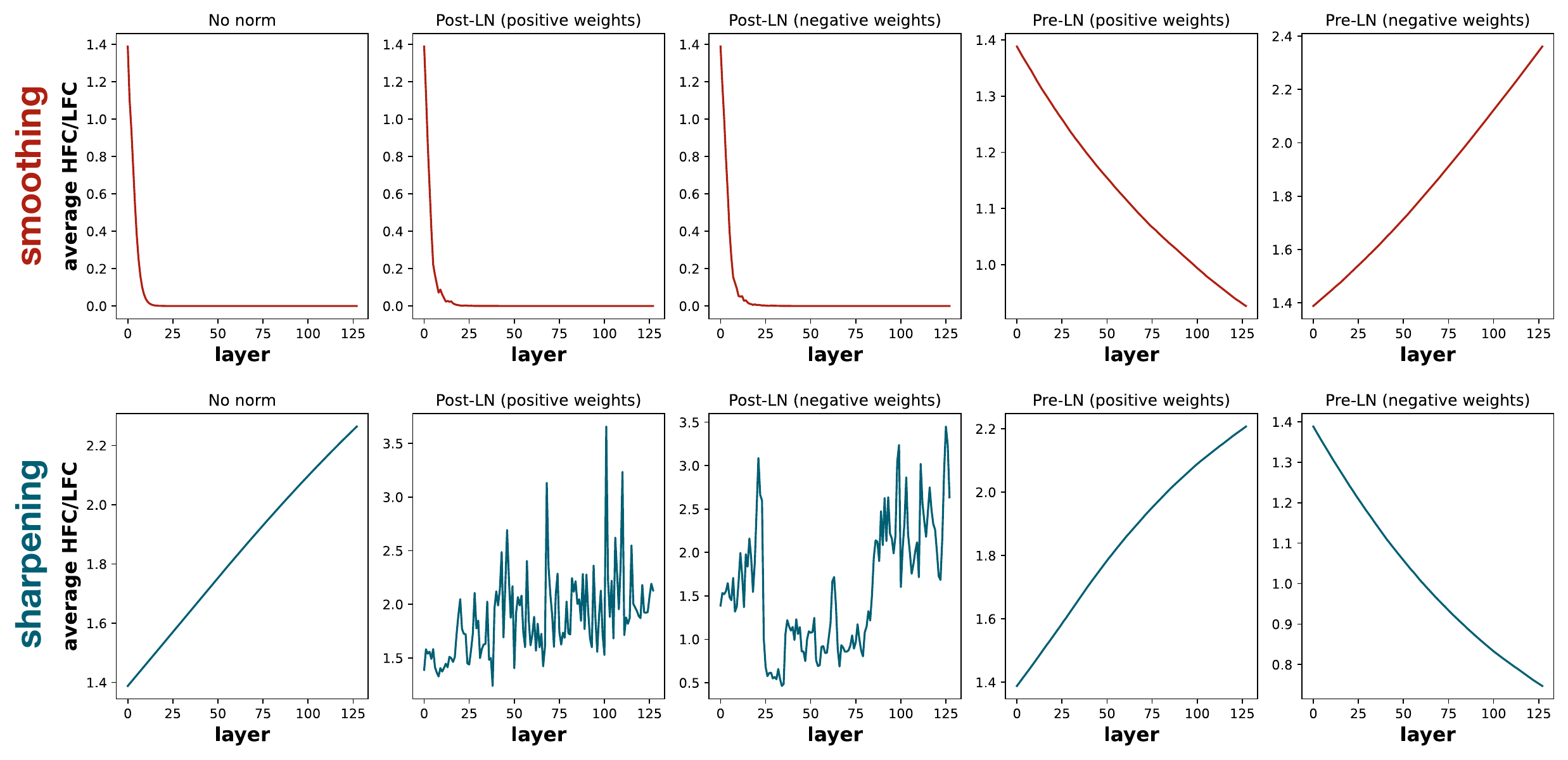}
\caption{\textbf{Impact of Layer Normalization.} The average $\mathrm{HFC}/\mathrm{LFC}$ for the Transformer update with repeated layers eq.~(\ref{eq:vec_update}) and different types of layer normalization (Post-LN \citep{vaswani2017attention}, Pre-LN \citep{baevski2018adaptive}) where the weights of the layer normalization are fixed to be positive or negative. See text for details.}
\label{fig:ln_impact}
\end{figure}

%% file: tables/hfclfc_change.tex
\begin{table*}[t!]
        \centering
        \resizebox{\textwidth}{!}{
        \begin{tabular}{c|ccc|ccccccc|ccc}
        \toprule
           & \multicolumn{3}{c|}{CIFAR100} & \multicolumn{7}{c|}{ImageNet} & \multicolumn{3}{c}{The Pile} \\ 
           \midrule
layer & ViT-Ti  & \textcolor{MidnightGreen}{ViT-Ti (sharpening)}  & \textcolor{Rufous}{ViT-Ti (smoothing)}  & ViT-S  & \textcolor{MidnightGreen}{ViT-S (sharpening)} & \textcolor{Rufous}{ViT-S (smoothing)} & ViT-B  & DeiT-B & ViT-L  & DeiT3-L & Cram. Bert & \textcolor{MidnightGreen}{Cram. Bert (sharpening)} & \textcolor{Rufous}{Cram. Bert (smoothing)} \\ 
\midrule
LayerNorm  & -0.073  & +0.011   & -0.043   & -0.126 & +0.276 & -0.244 & +0.157 & +0.338 & -0.364 & +0.811  & -0.915 & -0.086 & -0.019 \\
Attention  & -0.165  & +0.418   & -0.121   & -0.123 & -0.048 & +0.043 & -0.535 & -0.961 & +0.008 & -1.012  & +0.994 & -0.042 & -0.003 \\
MLP        & +0.425  & +0.418   & +0.168   & +0.175 & -0.496 & +0.270 & +0.061 & +0.258 & +0.624 & -0.604  & +0.914 & +0.316 & +0.043 \\ 
\bottomrule
\end{tabular}}
\caption{Change in $\mathrm{HFC}/\mathrm{LFC}$ for each layer type, across all models.} 
\label{tab:hfclfcchange}
\end{table*}

%% file: tables/erank_change.tex
\begin{table*}[t!]
        \centering
        \resizebox{\textwidth}{!}{
        \begin{tabular}{c|ccc|ccccccc|ccc}
        \toprule
           & \multicolumn{3}{c|}{CIFAR100} & \multicolumn{7}{c|}{ImageNet} & \multicolumn{3}{c}{The Pile} \\ 
           \midrule
Layer type & ViT-Ti  & \textcolor{MidnightGreen}{ViT-Ti (sharpening)}  & \textcolor{Rufous}{ViT-Ti (smoothing)}  & ViT-S  & \textcolor{MidnightGreen}{ViT-S (sharpening)} & \textcolor{Rufous}{ViT-S (smoothing)} & ViT-B  & DeiT-B & ViT-L  & DeiT3-L & Cram. Bert & \textcolor{MidnightGreen}{Cram. Bert (sharpening)} & \textcolor{Rufous}{Cram. Bert (smoothing)} \\ 
\midrule
LayerNorm  & +0.573 & +1.304 & +0.684  & +14.975 & +9.447 & +2.628  & +10.436 & +12.41 & +17.746 & +19.18  & +6.088 & +6.084 & +5.027 \\
Attention  & -0.171 & +4.754 & -2.870   & -15.454 & -5.185 & +6.203  & -10.247 & -14.301 & -18.671 & -16.939 & -5.927 & -6.338 & -5.002 \\
MLP        & +5.171 & -0.217 & +3.118   & -13.352 & -17.056 & -8.405 & -10.298 & -8.821 & -17.308 & -21.052  & -6.541 & -6.093 & -5.481 \\ 
\bottomrule
\end{tabular}}

\caption{Change in effective rank for each layer type, across all models.} 
\label{tab:erankchange}
\end{table*}

%% file: tables/cosine_similarity_change.tex
\begin{table*}[t!]
        \centering
        \resizebox{\textwidth}{!}{
        \begin{tabular}{c|ccc|ccccccc|ccc}
        \toprule
           & \multicolumn{3}{c|}{CIFAR100} & \multicolumn{7}{c|}{ImageNet} & \multicolumn{3}{c}{The Pile} \\ 
           \midrule
Layer type & ViT-Ti  & \textcolor{MidnightGreen}{ViT-Ti (sharpening)}  & \textcolor{Rufous}{ViT-Ti (smoothing)}  & ViT-S  & \textcolor{MidnightGreen}{ViT-S (sharpening)} & \textcolor{Rufous}{ViT-S (smoothing)} & ViT-B  & DeiT-B & ViT-L  & DeiT3-L & Cram. Bert & \textcolor{MidnightGreen}{Cram. Bert (sharpening)} & \textcolor{Rufous}{Cram. Bert (smoothing)} \\ 
\midrule
LayerNorm  & +0.006 & -0.004 & +0.001  & -0.057 & -0.013 & -0.056  & -0.061 & -0.064 & +0.002 & -0.166  & -0.274 & -0.232 & -0.025 \\
Attention  & +0.04 & -0.086 & +0.052  & +0.129 & +0.054 & -0.005  & +0.163 & +0.192 & +0.072 & +0.211 & +0.263 & +0.258 & +0.035 \\
MLP        & -0.078 & +0.054 & -0.038  & +0.258 & +0.021 & +0.021  & +0.005 & -0.052 & -0.058 & +0.117 & +0.111 & +0.188 & +0.016 \\ 
\bottomrule
\end{tabular}}

\caption{Change in cosine similarity for each layer type, across all models.} 
\label{tab:cossimchange}
\end{table*}

%% file: figures/eigen_distrib.tex
\begin{figure*}[h]
\centering
\includegraphics[width=0.8\textwidth]{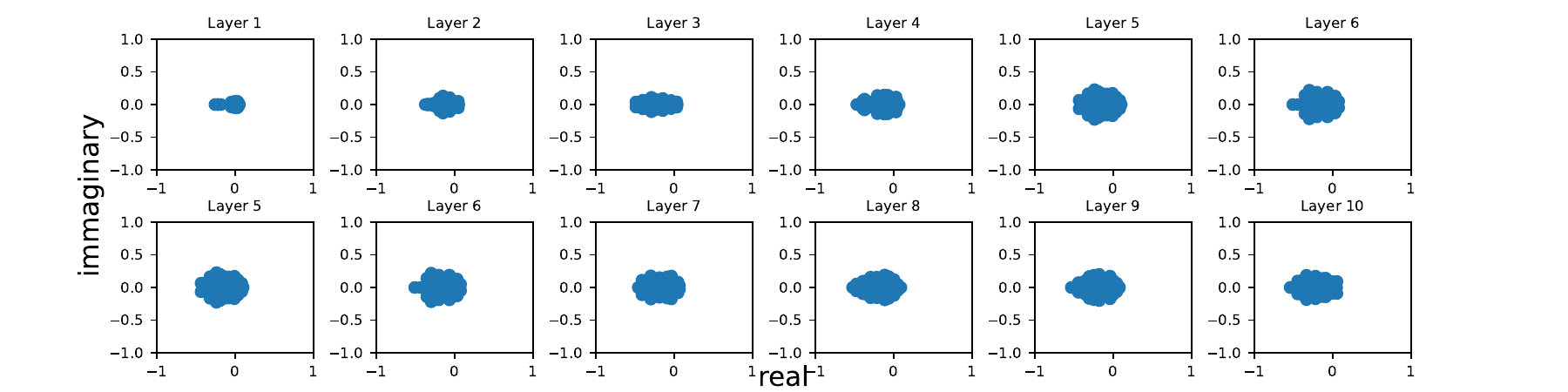}
\vspace{2ex}
\includegraphics[width=0.8\textwidth]{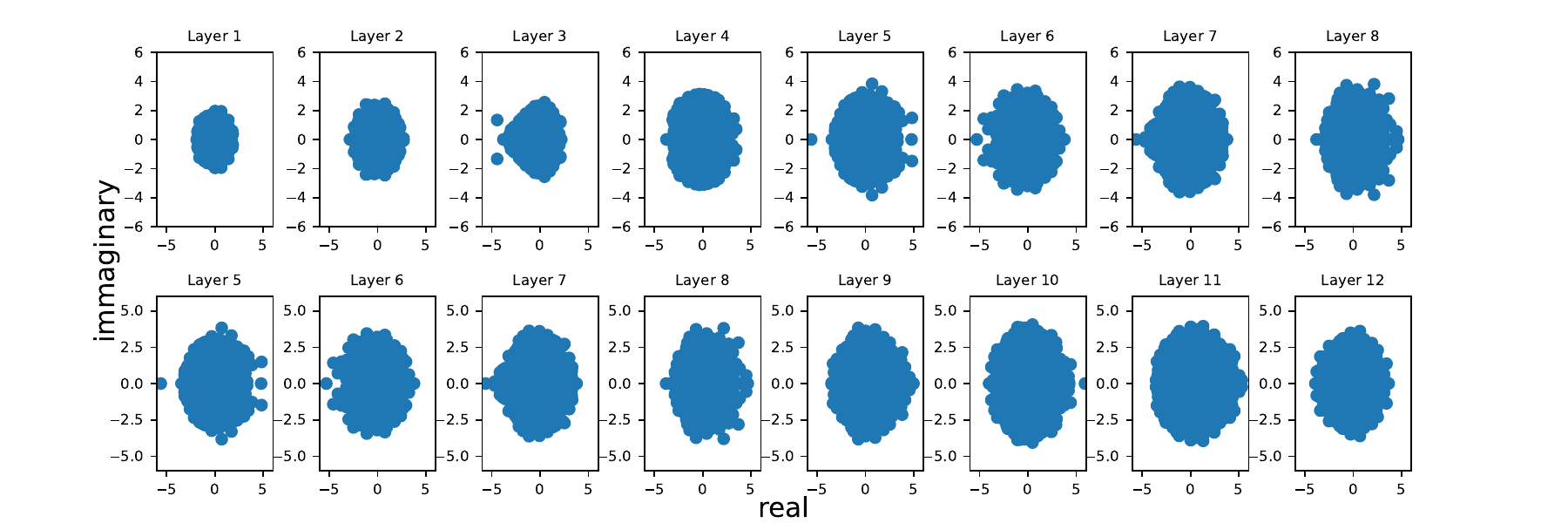}
\caption{\textbf{Distributions of eigenvalues of $\rmH$} (\emph{Top}) Vision models have distributions skewing to the negatives; (\emph{Bottom}) Language models have symmetrically distributed eigenvalues.}
\label{fig:eigen_distrib}
\end{figure*}